\documentclass[manuscript]{acmart}
\makeatletter
\renewcommand\@formatdoi[1]{\ignorespaces}
\makeatother

\usepackage{alias}

\usepackage{multirow}
\usepackage{flushend}
\usepackage{mathtools}
\usepackage[acronym]{glossaries}
\usepackage{units}
\usepackage{booktabs}
\usepackage[inline]{enumitem} 
\usepackage{subcaption}

\usepackage{arydshln}
\makeatletter
\def\adl@drawiv#1#2#3{%
        \hskip.5\tabcolsep
        \xleaders#3{#2.5\@tempdimb #1{1}#2.5\@tempdimb}%
                #2\z@ plus1fil minus1fil\relax
        \hskip.5\tabcolsep}
\newcommand{\cdashlinelr}[1]{%
  \noalign{\vskip\aboverulesep
           \global\let\@dashdrawstore\adl@draw
           \global\let\adl@draw\adl@drawiv}
  \cdashline{#1}
  \noalign{\global\let\adl@draw\@dashdrawstore
           \vskip\belowrulesep}}
\makeatother

\setcopyright{none}

\acmConference[REVEAL '20]{the ACM RecSys Workshop on Reinforcement Learning and Robust Estimators for Recommendation Systems}{}{}
\acmYear{2020}
\copyrightyear{2020}

\begin{document}
\title{Improving Offline Contextual Bandits with Distributional Robustness}

\author{Otmane Sakhi}
\authornote{Equal contribution.}
\email{o.sakhi@criteo.com}
\author{Louis Faury}
\authornotemark[1]
\email{l.faury@criteo.com}
\author{Flavian Vasile}
\affiliation{%
  \institution{Criteo AI Lab}
  \city{Paris}
  \country{France.}
}

\email{}

\begin{abstract}
This paper extends the Distributionally Robust Optimization (DRO) approach for offline contextual bandits laid out in \cite{faury20distributionally}. Specifically, we leverage this framework to introduce a \emph{convex} reformulation of the Counterfactual Risk Minimization principle introduced in \cite{swaminathan2015batch}. Besides relying on convex programs, our approach is compatible with stochastic optimization, and can therefore be readily adapted to the large data regime. Our approach relies on the construction of asymptotic confidence intervals for offline contextual bandits through the DRO framework. By leveraging known asymptotic results of robust estimators, we also show how to automatically calibrate such confidence intervals, which in turn removes the burden of hyper-parameter selection for policy optimization. We present preliminary empirical results supporting the effectiveness of our approach.
\end{abstract}

%
%

\maketitle

\section{Introduction}
\paragraph{Contextual Bandits.} The Contextual Bandit (CB) framework is a formalization of an important sequential decision making problem, with notorious applications in recommender systems \citep{li2010contextual,valko2014spectral}, mobile health \citep{tewari2017ads} and clinical trials \citep{villar2015multi}. It describes a repeated game between a decision-maker and an environment. The latter sequentially reveal sets of available actions to the former, along with some additional side information (or \emph{context}). Such additional information is assumed to carry informative signal about the intrinsic values (or \emph{rewards}) of the actions. Informally, the goal of the decision-maker is to discover an efficient strategy (or \emph{policy}) to select, given a context, a nearly optimal action.
   
\paragraph{Batch Learning from Bandit Feedback.} The CB optimization literature can be divided into two streams. The first studies its \emph{online} formulation, where the focus lies on the exploration/exploitation trade-off for \emph{regret} minimization. This paper is concerned with the second, known as Batch Learning from Bandit Feedback \citep{swaminathan2015batch} (BLBF) and arguably better suited for applications in real-life situations. The optimization is performed \emph{offline} and based on historical data, typically obtained by logging the interactions between an older version of the current policy and the environment. The learning problem consists in leveraging this data (necessarily biased towards actions favored by the logged policy) to discover new strategies of greater performance.
    
\paragraph{Prior work and limitations.} The first step in addressing the BLBF learning problem is to remove the intrinsic bias introduced by the logging policy \citep{bottou2015counterfactual}. This however can come at the price of building high-variance estimates \citep{swaminathan2015batch} for the performance of the current policy, which in turns can lead to high \emph{post-decision} regret. To address such challenges, the authors of \cite{swaminathan2015batch} introduced the Counterfactual Risk Minimization (CRM) principle. It combines debiasing through importance re-weighting \cite{swaminathan2015batch} with a modified policy-selection process that penalizes policies with high-variance estimates. Recently, \citep{faury20distributionally} proposed a generalization of the CRM principle through the Distributionally Robust Optimization (DRO) framework. This led them to the development of a new BLBF algorithm, obtained through a specialization of this general framework . However, while the methods introduced in both \citep{swaminathan2015batch} and \citep{faury20distributionally} offer some desirable theoretical guarantees for off-line policy optimization, their respective implementation suffer from important caveats. Namely, they rely on optimizing \emph{non-convex} objectives which is notably hard from a theoretical perspective. Further, these objective are not well-suited for stochastic optimization (useful when the logged data is large and cannot fit in memory), as obtaining unbiased stochastic gradients for these objective is not straight-forward. Finally, they rely on the selection of rather sensitive hyper-parameters, of which the (approximate) automatic calibration (through asymptotic arguments, for instance) is unknown.

\paragraph{Contributions.} In this paper, we further investigate the DRO framework for offline CB introduced in \citep{faury20distributionally}. This leads to a reformulation of the CRM principle that boils down to solving a convex problem. Further, we exploit some recent results from the stochastic optimization literature \citep{namkoong2016stochastic} to show how to efficiently solve this objective for large logged datasets. Our approach relies on the construction of asymptotic confidence intervals for offline CB through the DRO framework. Leveraging known asymptotic results for DRO we show how to automatically calibrate such confidence intervals, which in turns remove the need for hyper-parameter optimization for policy optimization. We validate our approach through extensive simulations on standard datasets for this task.

\section{Preliminaries}
\paragraph{Notations}In the following $\varphi$ is a real-valued convex function. The notation $d_\varphi$ refers to
the f-divergence associated to $\varphi$, and we denote $\varphi^\star(s)=\sup_{x\in\mbb{R}}(xs-\varphi(x))$ the Fenchel conjugate of $\varphi$. We will write $1_n$ to be the $n$-dimensional vector which entries are all equal to $1/n$. For
any positive integer $m$, $\Delta_m$ denotes the $m$-dimensional simplex.

\paragraph{Setting}
In the following, we will use $x\in\mcal{X}$ to denote a context and $a\in[K]$ an action, where $K$ denotes the number of available actions. Given a context $x$, each action is associated with a cost\footnote{We make this assumption for ease of exposition. It can be explicitly enforced by re-scaling the cost function.} $c(x,a)\in[-1,0]$, with the convention that better actions have smaller cost. The cost function $c$ is unknown. A decision maker is represented by its policy $\pi$ which maps each context $x\in\mcal{X}$ to $\Delta_K$. Assuming that the contexts are stochastic and follow a unknown distribution $\nu$, we define the \emph{risk} of the policy $\pi$ as the expected cost one suffers when playing actions according to $\pi$:
    \begin{align*}
        \risk{\pi} = \mbb{E}_{x\sim\nu, a\sim\pi(x)}\left[c(x,a)\right].
    \end{align*}
    The learning problem is to find a policy $\pi$ with smallest risk. In most real world problems, it is not reasonable to expect having the luxury of testing out several policies to compare their empirical risk and retain whichever policy has the smallest. This issue is usually circumvented by forecasting the risk of a given policy thanks to some existing interaction data. This is formalized through a \emph{logging policy} $\pi_0$ which has already been deployed in the environment (e.g a previous version of a recommender system that the practitioner is trying to improved) for which we assume we have the history $\mcal{H}_n= \left\{x_i,a_i\sim\pi_0(x_i), \pi_0(x_i,a_i), c(a_i,x_i)\right\}_{i\in[n]}$.
    Based on this data, one can build an unbiased (under mild assumptions) estimator of the risk of any policy $\pi$ through the use of importance weights \citep{rosenbaum1983central}:
    \begin{align*}
        \ipsrisk{\pi} := \frac{1}{n}\sum_{i=1}^n \omega_\pi(x_i,a_i)c(x_i,a_i) \quad \text{where }\quad w_\pi(x,a):= \pi(x,a)/\pi_0(x,a),
    \end{align*}
    commonly referred to as the IPS (Inverse Propensity Scoring) risk. 

\paragraph{Counterfactual Risk Minimization}
Unfortunately, the IPS estimator has potentially high variance,  depending on the disparity between $\pi$ and $\pi_0$ \citep[Section 4]{swaminathan2015batch}. Hence, directly sorting candidate policies thanks to their IPS risk is hazardous (as it boils down to comparing estimators with potentially high and different variances) and is known to be sub-optimal. To avoid this caveat, \cite{swaminathan2015batch} proposed to add an \emph{empirical variance} term to the IPS risk in order to penalize policies with high-variance estimates. Coined Counterfactual Risk Minimization (CRM), this principle suggest optimizing for the policy which minimizes:
    \begin{align}
    \label{eq:risk_crm}
        \poemrisk{\pi}= \ipsrisk{\pi} + \lambda\sqrt{\frac{\widehat{\text{Var}}_n(\pi)}{n}}\, ,
    \end{align}
    where $\lambda$ is a tunable hyper-parameter and $\widehat{\text{Var}}_n(\pi)=\frac{1}{n-1}\sum_{i=1}^n \Big(\omega_\pi(x_i,a_i)c(x_i,a_i)-\ipsrisk{\pi}\Big)^2$ is the empirical variance of $\ipsrisk{\pi}$. This policy selection process is based on \emph{variance-sensitive} confidence intervals for the true risk obtained via empirical Bernstein bounds \citep{maurer2009empirical}.

\paragraph{Generalization through DRO}
    Based on a similar intuition, \cite{faury20distributionally} recently introduced the idea of using Distributionally Robust Optimization (DRO) tools for this policy optimization problem. Formally, they showed that for a particular class of $\varphi$-divergence, the robust risk:
    \begin{align}
       \robustrisk{\varphi}{\pi}{\epsilon}:= \sup_{q\in\Delta_n}\left\{\sum_{i=1}^n q_i\omega_\pi(x_i,a_i)c(x_i,a_i)\quad \text{s.t} \quad d_\varphi(q,1_n)\leq \epsilon\right\}
       \label{eq:robust_risk}
    \end{align}
    is a  variance-sensitive (asymptotic) upper-bound for the true risk. It is therefore well-suited for the policy optimization task, and actually generalizes (in some sense) the CRM approach \citep[Lemma 3]{faury20distributionally}.
    For the KL-divergence, the robust risk has a closed-formed which can be directly minimized \cite[Lemma 4]{faury20distributionally}:
    \begin{align*}
     \robustrisk{\text{KL}}{\pi}{\epsilon}= \sum_{i=1}^n \frac{\exp(\omega_\pi(x_i,a_i)c(x_i,a_i)/\gamma)}{\sum_j \exp(\omega_\pi(x_j,a_j)c(x_j,a_j)/\gamma)}\omega_\pi(x_i,a_i)c(x_i,a_i)
    \end{align*}
 with  $\gamma$ being a tunable hyper-parameter. 

\paragraph{Limitations and contributions}
The variance penalization of the IPS objective and its DRO generalization benefit from solid theoretical justifications, and result in better policy optimization algorithms. These algorithms, that were proven to outperform the simple IPS objective \cite{swaminathan2015batch, faury20distributionally} can still be improved as they suffer from important limitations: \textbf{(1)} Contrary to the IPS objective, which is linear (and therefore convex) in $\pi$, both the initial CRM objective (adding a square root variance penalization) and the DRO objective (as it was introduced in \cite{faury20distributionally}) break this convexity. This results in ill-posed optimization programs, which potentially hinders the statistical benefits brought by such methods. \textbf{(2)} Another important limitation of such objectives is their scalability to large datasets. Both formulations are not well-adapted to stochastic gradient descent algorithms, as obtaining unbiased stochastic gradients of their related objectives is not-straightforward. For instance, the algorithm introduced in \cite{faury20distributionally} works only in the batch setting as the adversary distribution needs all the data to be normalized. \cite{swaminathan2015batch} suggested a relaxation of the CRM objective amenable to stochastic gradients, however only applicable in the case of exponential policies. Their approach consists in a majorization-minimization strategy, and still requires to go through the whole logged dataset once in a while. \textbf{(3)} These algorithms also come with hyper-parameters that need careful tuning as their choice drastically impact the performance of the obtained policy, making the whole optimization procedure even harder. \cite{swaminathan2015batch} treats $\lambda$, the weight of the variance penalty, as a hyper-parameter, while \cite{faury20distributionally} treat $\epsilon$, the maximum distance between the adversarial and the nominal distribution as a hyper-parameter as well.
To sum-up, both algorithms are deemed rather impractical in real life, as they have a non convex loss surface to optimize, are not applicable to huge datasets and need consequent hyper-parameter tuning. This work tries to circumvent these limitations through a more careful treatment of the DRO formulation, leading to general algorithms that treat all these caveats in a well defined and unified framework.

\section{Policy Evaluation and Optimization}
\subsection{Policy Evaluation: Confidence Intervals}

\begin{table}[t]
    \centering
    \begin{tabular}{|c|c|c|c|}
    \hline
         \textbf{Divergence} & $\pmb{\varphi(t)}$ & $\pmb{D_\varphi(q\vert\vert p)}$ & $\pmb{\varphi^*(s)}$ \\
         \hline
         Chi-Square & $(t-1)^2$& $\sum_{i=1}^n \frac{(q_i-p_i)^2}{p_i}$ & $\left\{\begin{aligned}
             &s + s^2/4 \quad &s\geq -2\\
             &-1 \quad &s\leq -2
             \end{aligned}
         \right.$\\
         \hline
         Kullback-Leibler & $t\log t -t +1$ & $\sum_{i=1}^n q_i\log(q_i/p_i)$& $e^s-1$ \\
         \hline
         Burg entropy & $-\log t + t -1$& $\sum_{i=1}^n p_i\log(p_i/q_i)$ & $-\log(1-s)$, $s<1$ \\
         \hline
         Hellinger distance & $(\sqrt{t}-1)^2$& $\sum_{i=1}^n \left(\sqrt{p_i}-\sqrt{q_i}\right)^2$ & $\frac{s}{1-s}$, $s\leq 1$  \\
         \hline
    \end{tabular}
    \caption{Some coherent $\varphi$-divergences and their characterizations.}
    \label{tab:phi_divs}
\end{table}

In this section, we briefly review and discuss how the robust risk can lead to the construction of confidence intervals for the true risk. This is a crucial step for policy optimization, as the latter will consist in minimizing the high-probability upper-bound on the true risk provided by the policy evaluation procedure. Designing tight confidence interval for the risk can also be a goal in itself, in order to fully evaluate the potential benefits/risks of deploying a given policy (e.g provide an offline metric for A/B testing). In the following, we follow \cite{faury20distributionally} and consider \emph{coherent} $\varphi$-divergence - i.e we will assume that $\varphi$ satisfies the conditions of Assumption 1 in \cite{faury20distributionally}. We provide some examples of such functions (and their associated divergence measure) in Table
   ~\ref{tab:phi_divs}. 
   
Under such conditions, one can show that the robust risk accounts for the variance of the IPS risk estimator \cite[Lemma 2]{faury20distributionally}. Further, DRO can be further leveraged to build asymptotic confidence intervals for the robust-risk. Indeed, let us introduce the following problem, converse to \eqref{eq:robust_risk}:
   \begin{align}
    \optimisticrisk{\varphi}{\pi}{\epsilon} := \inf_{q\in\Delta_n}\left\{\sum_{i=1}^n q_i\omega_\pi(x_i,a_i)c(x_i,a_i) \quad \text{s.t} \quad d_\varphi(q,1_n)\leq \epsilon\right\}.
    \label{eq:optimistic_risk}
   \end{align}
   We have the following result, extending Lemma 1 of \cite{faury20distributionally} and easily extracted from \cite{duchi2016statistics}. 
   \begin{restatable}{prop}{propci}[Asymptotic Confidence Interval]
   \label{prop:propci}
    Let $\delta\in[0,1)$. For $\alpha\in(0,1)$ denote $\rho_\alpha$ the $(1-\alpha)$-quantile of the one-dimensional $\chi^2$ distribution. Then:
    \begin{align}
    \label{eq:conf_interval}
        \lim_{n\to\infty} \mathbb{P}\left(\optimisticrisk{\varphi}{\pi}{\frac{\rho_{\delta}}{n}} \leq \risk{\pi} \leq \robustrisk{\varphi}{\pi}{\frac{\rho_{\delta}}{n}}\right) \geq 1-\delta\, .
    \end{align}
   \end{restatable}
   This result states that the interval $[\optimisticrisk{\varphi}{\pi}{\epsilon},\robustrisk{\varphi}{\pi}{\epsilon}]$ is a asymptotic $(1-\delta)$ confidence interval for the true risk, when the size of the ambiguity-set $\epsilon$ is set to $\rho_{\delta}/n$. We will show in Section~\ref{sec:exp_pe} that despite being asymptotic, this interval is empirically \emph{tight} and displays satisfying coverage, motivating its use in real-life applications. It turns out that the programs for computing the robust and optimistic risk  (Equation~\eqref{eq:robust_risk} and \eqref{eq:optimistic_risk}, respectively) can be efficiently solved. We will only review here the computation of the robust risk, however a similar reasoning holds for the optimistic risk. Notice that the objective in Equation~\eqref{eq:robust_risk} is linear in the variable $q$, which acts as a re-weighting for the counterfactual costs. Further, the constraint set $\{q\in\Delta_n \,|\, d_\varphi(q,1_n)\leq \epsilon\}$ is convex. The program is therefore \emph{convex} and can henceforth be solved efficiently. In this paper, we will rely on its dual formulation, which is easily solvable and well-adapted to the stochastic setting. Formally, we rely on the following result to characterize the robust risk.
   \vspace{-10pt}
    \begin{restatable}{lemma}{lemmadual}[Dual program for the robust risk]
    \label{lemma:lemmadual}
        Let:
        \begin{align}
        \label{eq:defg}
            g_\pi(\beta,\gamma) = \beta + \gamma\epsilon + \frac{1}{n}\sum_{s=1}^{n}(\gamma\varphi)^\star\left(\omega_\pi(x_i,a_i)c(a_i,x_i)-\beta\right)\, ,
        \end{align}
        where $(\gamma\varphi)^\star(s)=\gamma\varphi^\star(s/\gamma)$, with the convention that $(0\varphi)^\star(s)=+\infty$ is $s>0$ and $0$ otherwise. The function $(\pi,\beta,\gamma)\to g_\pi(\beta,\gamma)$ is convex and:
        \begin{align}
        \label{eq:dro_pe}
           \robustrisk{\varphi}{\pi}{\epsilon} = \inf_{\beta,\gamma\geq 0} g_\pi(\beta,\gamma)\, .
           \tag{\textbf{DRO-PE}}
        \end{align}
    \end{restatable}
    This robust program characterization can be extracted from more general results - see for instance \citep[Section 4]{ben2013robust}. We provide a detailed proof in Appendix~\ref{app:prooflemmadual} for the sake of completeness. In a few words, Lemma~\ref{lemma:lemmadual} states the the robust risk can be efficiently computed by solving a two-dimensional convex program. When $n$ is reasonably small (in other words, when the dataset $\mcal{H}_n$ fits in memory), coordinate descent (with exact line search) or two-dimensional bisection provide efficient, principled tools for computing the robust risk. The program \eqref{eq:dro_pe} is also well-suited for the large-data regime (\emph{e.g} large $n$) as it naturally adapts to stochastic optimization. Indeed, the function $g$ (Equation \eqref{eq:defg}) is composite and unbiased gradients of this objective are easily obtainable. Stochastic gradient descent methods (see \citep{ruder2016overview} for a modern overview) therefore provide efficient and flexible solutions for this problem (up to some mild modifications to account for the fact that the $g$ is not smooth for $\gamma$ in a neighborhood of $0$).

To sum-up, we showed here how the DRO method could be used to build confidence intervals for the true risk, by simply relying on solving convex programs. This confidence intervals are however asymptotic; we will show in Section~\ref{sec:exps} that, still, they provide sufficient coverage, while being much tighter than their finite-time counterparts. 
    
\subsection{Policy Optimization: Towards a Convex Objective}
\subsubsection{General principle}
The CRM principle casts policy optimization in a theoretically sound framework, and interestingly enough is closely related to the robust risk defined throughout the paper.
Relying on \cite{duchi2016statistics}, \cite{faury20distributionally} namely showed that the robust risk provides an asymptotic approximation to the variance regularized empirical risk.
 The original CRM principle for policy optimization (Equation~\eqref{eq:risk_crm}) can therefore be rethought as the minimization of a $\varphi$-robust risk. Using the dual formulation of the robust risk given in Equation~\eqref{eq:dro_pe}, the policy optimization objective becomes:
\begin{align}
\label{eq:dro_po}
\inf_{\pi} \robustrisk{\varphi}{\pi}{\epsilon} = \inf_{\pi, \beta,\gamma\geq 0} g_\pi(\beta,\gamma)\, .
\tag{\textbf{DRO-PO}}
\end{align}

Note that as a consequence of Lemma~\ref{lemma:lemmadual}, this objective is \emph{convex}. It can therefore be minimized in principled ways, while enjoying similar guarantees as the original CRM objective. Intuitively, one can expect such an important transformation of the optimization properties of the policy improvement objective to lead to greater practical performances. The most natural way to solve the policy improvement objective \eqref{eq:dro_po} is through plain gradient descent. Indeed, a valid strategy consists in feeding gradients of the function $(\beta, \gamma,\pi)\to g_\pi(\beta,\gamma)$ to a gradient optimizer. In our experiments, we found that applying the L-BFGS solver to this dual program to work best. As for the policy evaluation, the policy optimization objective \eqref{eq:dro_po} is particularly adapted when the historic data is large (i.e $n\gg 1$) and only stochastic gradients can be obtained. For stochastic optimization, stochastic gradients methods can encounter some issues (linked to the possible unboundedness of the gradients) which can be alleviate thanks to specialized methods \cite{namkoong2016stochastic}. 

    \subsubsection{Extensions}
    In the following, we discuss how different estimators can be used and robustified in the same way as the IPS, for improved performances and without sacrificing convexity.
    
    \paragraph{Variance Reduction} The methods presented so far rely on vanilla IPS. It is well known that this estimator suffers from large variance which can lead to poor performances - whatever the policy optimization algorithm used. Fortunately, our method easily extend to other estimators, so long that they remain convex in $\pi$. Still, it would be useful to extend our method to estimators that actively reduce variance. A candidate for this task is the self-normalized importance sampling estimator of \citep{snips}. This estimator is unfortunately not convex in $\pi$, which goes against the efforts undertaken in this paper to maintain well-behaved optimization tasks. We provide here an alternative which uses a simple additive control variate (instead of a multiplicative one). Formally, we rely on the following estimator:
    \begin{align*}
        \text{Risk}_{n,\rho}(\pi) = \frac{1}{n}\sum_{i=1}^n \left(c(x_i,a_i)-\rho\right)\omega_\pi(x_i,a_i) + \rho \, .
    \end{align*}
    A robust version of this estimator easily follows, and enjoys the same convex properties of the IPS robust risk. The variance-reduction property of the additive control variate is presented in the following Lemma. 
    
     \begin{restatable}{lemma}{lemmaestimator}[Propensity weights as an additive control variate]
    \label{lemma:lemmaestimator}
        For all $\rho$, $\textnormal{Risk}_{n,\rho}(\pi)$  is an unbiased estimator of $\risk{\pi}$, achieving a better variance than naive IPS whenever $0 \le \rho \le 2\frac{\text{Cov}(\ell_\pi, \omega_\pi)}{\mathbb{\mathbf{V}}(\omega_\pi)}$. 
        In addition, if the cost is independent of the propensity weights, we obtain $\rho^* = \E[c]$.
    \end{restatable}

In practice, we do not know how to derive $\rho^*$ analytically, however one can directly use the cost's empirical mean under $\pi_0$.
    
\paragraph{Parametric Policies.} In practice, the actions/contexts space is extremely large and directly optimizing the objective with respect to the policy (as a $\mathbb{R}^{\vert \mcal{X}\vert\times K}$ matrix) is unreasonable. In such cases, policies are parametrized to drastically reduce the complexity of the problem. This usually breaks convexity, even in the simplest case of log-linear policies - that is, policies of the form $\pi_\theta(a|x) \propto \exp(\theta^T f(x,a))$ for $f(x,a)$ a given joint feature map. In this case, the objective becomes a negative sum of log-concave functions resulting in a non-convex optimization surface. Following \citep{logtrick}, one can bypass this non-convexity by constructing a \emph{tight} convex upper bound of the original objective.
     \begin{restatable}{lemma}{lemmaelogtrick}[Convex upper-bound for log-concave policies]
    \label{lemma:lemmalogtrick} 
     Let $\pi_\theta$ be a log-concave (w.r.t $\theta$) policy. For a given $\theta_0$, let:
     \begin{align*}
         \textnormal{Risk}_{n}^{\text{up}}(\pi_\theta) = \frac{1}{n}\sum_{i=1}^{n}\frac{\pi_{\theta_0}(a_i,x_i)}{\pi_0(a_i,x_i)}(1 + \log[\frac{\pi_{\theta}(a_i,x_i)}{\pi_{\theta_0}(a_i,x_i)}])c(a_i,x_i)\, .
     \end{align*}
     $\textnormal{Risk}_{n}^{\text{up}}(\pi_\theta)$ is a \emph{convex} upper bound of the IPS risk. The closer $\theta_0$ to $\theta$, the tighter the upper bound, with equality at $\theta_0 = \theta$.
    \end{restatable}

We can use Lemma \ref{lemma:lemmalogtrick} to obtain a proxy of our initial objective, building on an iterative procedure that only uses convex losses throughout the whole optimization process. Once again, we can build a robust version of this estimator which can be efficiently optimized. Note that here, the robust estimator will be convex w.r.t the parametrization $\theta$ as soon as the policy in log-concave.

\section{Preliminary Results}
\label{sec:exps}

We here describe some preliminary experimental results, backing up the idea that an improved optimization landscape for policy optimization naturally leads to improved practical performances. We work with the four $\varphi$-divergence presented in Table~\ref{tab:phi_divs}. We employ the classical supervised to bandit conversion \citep{agarwal2014taming}. Formally, denote $x\in\mathcal{X}$ a given input vector and $y\in\{0,1\}^L$ its label, and $\mathcal{D}
^\star= \{(x_1,t_1),\hdots,(x_m,t_m)\}
$ a given multi-label dataset. We create the logging policy by training it on a fraction of $\mathcal{D}^\star$. We then create the historic data $\mathcal{H}_n$ by going repeating $P$ times the following procedure: for every  $(x_i,t_i)$ in the supervised dataset, sample $a_i\sim\pi_0(x)$ and log the cost $c(a_i,x_i)=\lVert a_i - t_i\rVert_1$. Following \citep{swaminathan2015batch}, we call $P$ the \emph{replay count}.

\subsection{DRO Confidence Intervals}
\label{sec:exp_pe}

We start this experimental section by performing a \emph{sanity check} on the (asymptotic) confidence intervals that we based our policy optimization method on. Formally, we evaluate the finite-time validity of Equation~\eqref{eq:conf_interval}. Being asymptotic, we can safely expect DRO-based confidence intervals to be smaller than their finite-time counterparts - i.e confidence intervals based on Hoeffding or empirical Bernstein tail inequalities (see \citep{thomas2015high} for their application to policy evaluation). We however wish to check that they provide reasonable coverage in non-asymptotic regimes. To do so, we train a policy $\pi$ on a random subset of $\mathcal{D}
^\star$ and evaluate the empirical mean coverage and width of DRO-based intervals. In Figure~\ref{fig:pe}, we present such results on two datasets: Yeast and Scene, taken from the LibSVM repository and standard for the policy optimization task \citep{swaminathan2015batch}. The empirical coverageis reported for increasing values of the replay count $P$, or equivalently for increasing values of the historic data size $n$. The failure level is set to $\delta=0.95$ in all experiments. We observe that the DRO- ased confidence interval provide almost exact $(1-\delta)$ coverage, and it is therefore safe to use them even in the small data regime. As a side comment, we observe that all four $\varphi$-divergence lead to very similar results. Finally, we check experimentally that as expected, the asymptotic DRO-based confidence intervals are by orders of magnitude smaller than finite-time ones.

\begin{figure}[t]
\begin{subfigure}{0.44\linewidth}
    \includegraphics[width=\linewidth]{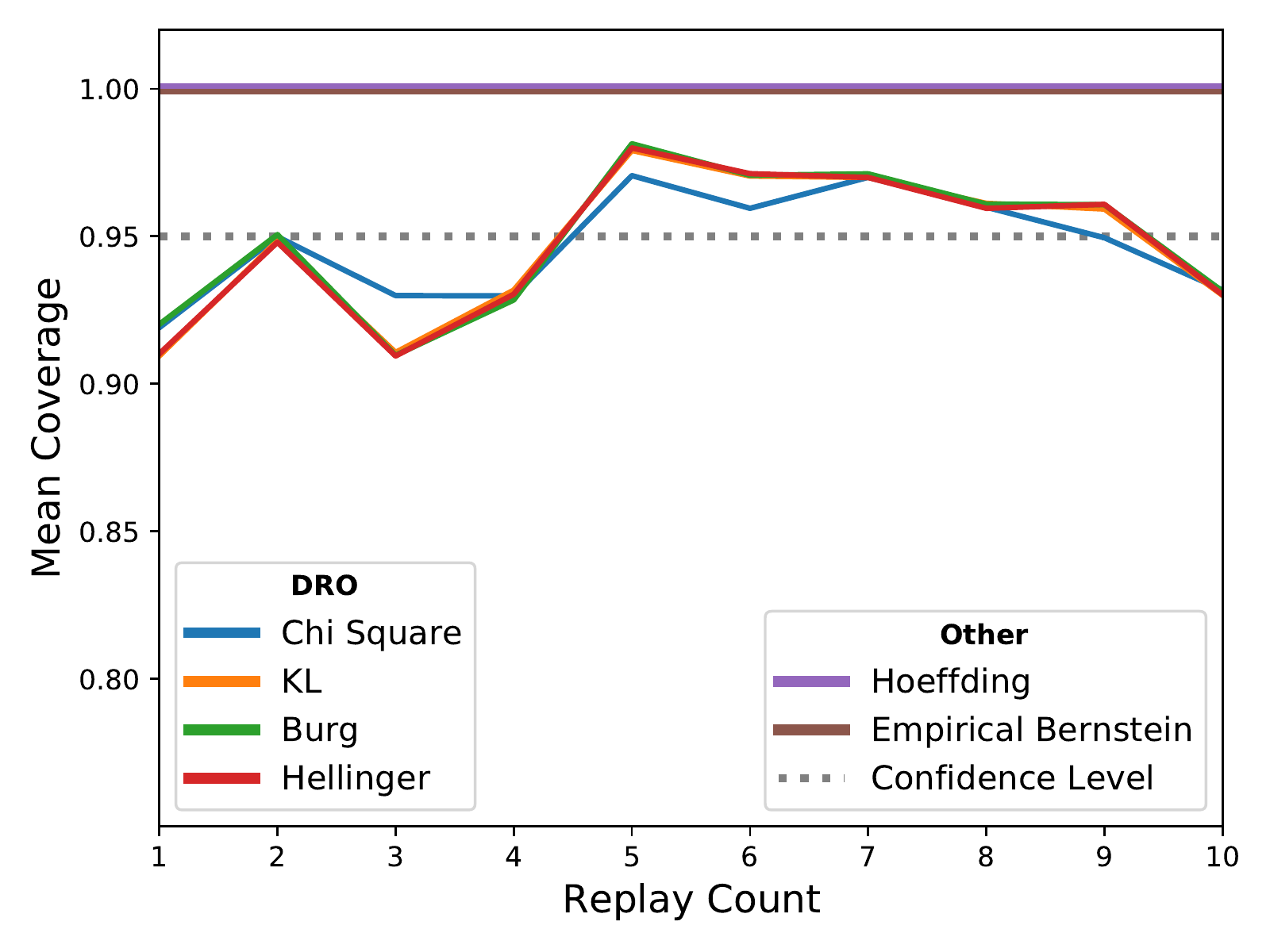}
    \caption{Empirical coverage of the true risk for different confidence intervals on the Yeast dataset.}
\end{subfigure}
\begin{subfigure}{0.44\linewidth}
\includegraphics[width=\linewidth]{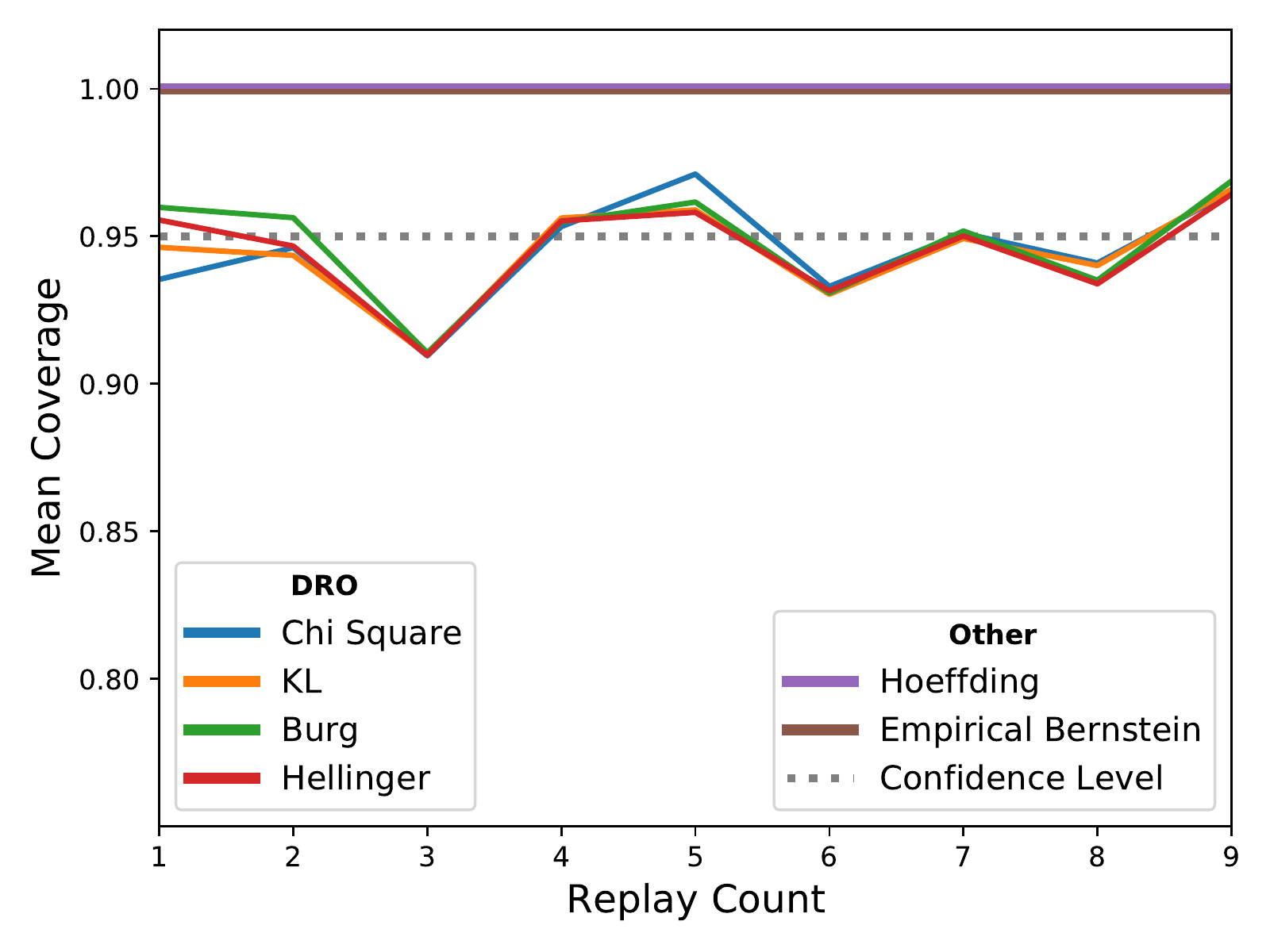}
\caption{Empirical coverage of the true risk for different confidence intervals on the Scene dataset.}
\end{subfigure}
\caption{Finite-time evaluation (coverage) for asymptotic DRO-based confidence intervals.}
\label{fig:pe}
\end{figure}

\subsection{Policy Optimization}
\label{sec:exp_po}

We report here some preliminary results for the policy optimization, for which we strictly follow the experimental procedure of \cite{swaminathan2015batch}. The supervised dataset $\mathcal{D}^\star$ is split into three parts (train, validation and test). A logging policy is train on a random fraction ($=0.1$) of $\mathcal{D}_*$ and used to collect the history $\mathcal{H}_n$ by running it through the training data $P=4$ times. For DRO-based algorithms, the validation set is not used, since no hyper-parameter needs to be tuned (we use the value recommended by the asymptotic analysis for $\varepsilon$, with a fixed confidence level at $\delta=0.05$). For POEM and its stochastic approximation, the parameter $\lambda$ is selected by cross-validation on the validation data. We report in Figure
~\ref{fig:res_po} the risks of the policy (and their greedy versions) returned by the different algorithms. As in \cite{swaminathan2015batch}, all policies are parametrized linearly, with a softmax output activation layer. We present results for both batch algorithms ($\mcal{B})$ and their stochastic versions  ($\mcal{S})$. Results are average over 20 random repetitions. 

\begin{figure}[t]
\begin{subfigure}{0.45\linewidth}
    \centering
    \begin{tabular}{|c||c|c|}
    \hline
        \textbf{Algorithm} & \textbf{Risk} & \textbf{Greedy-Risk} \\
        \Xhline{2\arrayrulewidth}
        POEM-$\mcal{B}$ & 0.93 (0.06) & 0.91 (0.06)\\
        DRO-$\mcal{B}$-$\chi^2$ & \textbf{0.89} (0.06) & 0.87 (0.06)\\
        DRO-$\mcal{B}$-KL & 0.90 (0.06) & 0.88 (0.06)\\
        DRO-$\mcal{B}$-Burg & 1.06 (0.06)  &\textbf{0.85} (0.05)\\
        DRO-$\mcal{B}$-Hellinger & 1.06 (0.06) & 0.85 (0.05) \\
   \hline
    POEM-$\mcal{S}$ & \textbf{1.0} (0.05) & \textbf{0.97} (0.05) \\ 
    DRO-$\mcal{S}$-$\chi^2$ & 1.05 (0.06) & 1.02 (0.06) \\
    DRO-$\mcal{S}$-KL & 1.06 (0.07) & 1.04 (0.07)\\
    DRO-$\mcal{S}$-Burg & 1.12 (0.05) & 1.08 (0.05) \\
    DRO-$\mcal{S}$-Hellinger & 1.3 (0.08) & 1.18 (0.07)\\
    \hline
    \end{tabular}
    \caption{Scene dataset. The quantity in parenthesis represent one standard deviation of the aggregated results.}
\end{subfigure}\hfill
\begin{subfigure}{0.45\linewidth}
    \begin{tabular}{|c||c|c|}
    \hline
        \textbf{Algorithm} & \textbf{Risk} & \textbf{Greedy-Risk} \\
        \Xhline{2\arrayrulewidth}
        POEM-$\mcal{B}$ & 5.15 (0.07) & 4.34 (0.13)\\
        DRO-$\mcal{B}$-$\chi^2$ & 5.21 (0.05) & 4.38 (0.12)\\
        DRO-$\mcal{B}$-KL & 5.32 (0.04) & 5.29 (0.11)\\
        DRO-$\mcal{B}$-Burg & \textbf{4.77} (0.07)  & \textbf{3.74} (0.13)\\
        DRO-$\mcal{B}$-Hellinger & \textbf{4.77} (0.07) & \textbf{3.74} (0.13) \\
   \hline
    POEM-$\mcal{S}$ &  \textbf{5.16} (0.05) & \textbf{4.62} (0.1)\\ 
    DRO-$\mcal{S}$-$\chi^2$ & 5.17 (0.06) & 4.71 (0.11) \\
    DRO-$\mcal{S}$-KL & 5.17  (0.06) & 4.72 ( 0.12)\\
    DRO-$\mcal{S}$-Burg & 5.17 (0.06) & 4.72 (0.1) \\
    DRO-$\mcal{S}$-Hellinger & 5.27 (0.06) & 4.71 (0.1)\\
    \hline
    \end{tabular}
    \caption{Yeast dataset. The quantity in parenthesis represent one standard deviation of the aggregated results.}
\end{subfigure}
\caption{Policy optimization results.}
\label{fig:res_po}
\end{figure}

\section{Discussion}

For batch algorithms, one can notice that DRO-based methods provide either similar or better empirical results than POEM on both considered datasets, while being hyper-parameter free (which is not the case of POEM). On the Yeast dataset, the improvement is quite significative for two of the four $\varphi$-divergence (Burg and Hellinger). On the negative side, it seems there is no consistency in the relative performance of the different divergences. This is quite troublesome in practice, as to the best of our knowledge there is no obvious nor preferable choice of divergences given a dataset.  A solution to this problem is probably to cross-validate this choice, potentially over a continuous parametrization of the divergence considered here (such as the parameter of a Cressie-Read divergence). Finally, we note that POEM-$\mcal{S}$ dominates among the stochastic algorithm considered. This is however to be nuanced, as this algorithm still needs to load in memory the entire dataset at every epoch (\emph{e.g.} every time an upper-bound on the true objective is constructed). This is not the case for DRO-based algorithms. We also believe that the nonetheless good performances reported here for stochastic DRO algorithms will turn to be decisive when considering more complex policies (\emph{e.g} parametrized by a neural network, where POEM-$\mcal{S}$ have been reported to fail). We plan on investigating this in future work, along with the performance of other robustified estimators (cf. Lemma
~\ref{lemma:lemmaestimator} and \ref{lemma:lemmalogtrick}.)

\newpage
  
\bibliographystyle{ACM-Reference-Format}
\bibliography{bibliography}

\appendix

\section{Proof of Lemma~\ref{lemma:lemmadual}}
\label{app:prooflemmadual}
\lemmadual*

\begin{proof}
Recall the definition of the robust risk:
\begin{align}
    \label{eq:prog}
     \robustrisk{\varphi}{\pi}{\epsilon}:= \sup_{q\in\Delta_n}\left\{\sum_{i=1}^n q_i\omega_\pi(x_i,a_i)c(x_i,a_i)\quad \text{s.t} \quad d_\varphi(q,1_n)\leq \epsilon\right\}
     \tag{\textbf{P}}
\end{align}
where:
\begin{equation*}
    \left\{
    \begin{aligned}
        &\Delta_n = \left\{p\in\mbb{R}_n^+ \, \middle\vert \sum_{i=1}^n p_i = 1\right\}\\
        &1_n = \frac{1}{n}(1 \hdots 1)^\transp \in\mbb{R}^n\\
        &d_\varphi(q,p) = \sum_{i=1}^n p_i\varphi\left(\frac{q_i}{p_i}\right) \quad \forall q\ll p \in\Delta_n
    \end{aligned}
    \right.
\end{equation*}
Note that the program~\eqref{eq:prog} optimizes a linear objective under convex constraints (since $\varphi$ is convex). Further, when $\epsilon>0$, the candidate $q=1_n$ is \emph{strictly} feasible. Therefore, Slater's condition holds and \eqref{eq:prog} enjoys strong duality. Writing down its Lagrangian, we obtain the following equivalence:
\begin{align}
\label{eq:intermediary_def_g}
    \robustrisk{\varphi}{\pi}{\epsilon} &= \sup_{q\succeq 0} \inf_{\beta,\gamma\geq 0} \sum_{i=1}^n q_i\omega_\pi(x_i,a_i)c(x_i,a_i) + \beta\left(1-\sum_{i=1}^n q_i\right) + \gamma\left(\epsilon - \frac{1}{n}\sum_{i=1}^n \varphi(nq_i)\right)\notag \\
     &=\inf_{\beta,\gamma\geq 0} \sup_{q\succeq 0} \sum_{i=1}^n q_i\omega_\pi(x_i,a_i)c(x_i,a_i) + \beta\left(1-\sum_{i=1}^n q_i\right) + \gamma\left(\epsilon - \frac{1}{n}\sum_{i=1}^n \varphi(nq_i)\right)\notag\\
     &= \inf_{\beta,\gamma\geq 0} \beta+\gamma\epsilon  + \frac{1}{n}\sum_{i=1}^n \sup_{q_i\geq 0}\left\{(nq_i)\omega_\pi(x_i,a_i)c(x_i,a_i)-\gamma\varphi(nq_i)\right\}
\end{align}
where the first equality is a consequence of strong duality, and the second is obtained through simple re-arranging. If $\gamma\neq 0$, easy computations lead to:
\begin{align*}
    \robustrisk{\varphi}{\pi}{\epsilon} &= \inf_{\beta,\gamma\geq 0} \beta+\gamma\epsilon  + \frac{\gamma}{n}\sum_{i=1}^n \sup_{q_i\geq 0}\left\{(nq_i)\frac{\omega_\pi(x_i,a_i)c(x_i,a_i)}{\gamma}-\varphi(nq_i)\right\}\\
    &= \inf_{\beta,\gamma\geq 0} \beta+\gamma\epsilon  + \frac{\gamma}{n}\sum_{i=1}^n \varphi^\star\left(\frac{\omega_\pi(x_i,a_i)c(x_i,a_i)}{\gamma}\right)
\end{align*}
by using the definition of $\varphi^\star$. The limit conditions announced in the Lemma are easily checked by computing the dual function when $\gamma=0$. We therefore obtain the equality announced by using the definition of $g_\pi$:
\begin{align*}
    \robustrisk{\varphi}{\pi}{\epsilon} = \inf_{\beta,\gamma\geq 0} g_\pi(\beta,\gamma)
\end{align*}
The convexity of $g_\pi$ can be obtained two ways; (1) by noticing that $g_\pi$ is obtained through convexity-transforming transformations of a \emph{perspective} function \citep{combettes2018perspective}, or (2) by noticing thanks to Equation~\eqref{eq:intermediary_def_g} that:
\begin{align*}
    (\pi,\beta,\gamma) \to  \sum_{i=1}^n \sup_{q_i\geq 0}\left\{(nq_i)\omega_\pi(x_i,a_i)c(x_i,a_i)-\gamma\varphi(nq_i)\right\}
\end{align*}
is convex as a sum of supremum of linear (and hence convex) functions. 
\end{proof}





\end{document}